\documentclass{article} 
\usepackage{nips13submit_e,times}
\usepackage{hyperref}
\usepackage{url}
\usepackage{comment}
\usepackage{amsmath}
\usepackage{amssymb}
\usepackage{amsthm}
\usepackage{graphicx}

\newtheorem{theorem}{Theorem}
\newtheorem{lemma}{Lemma}

\title{Contextually Supervised Source Separation with Application to
  Energy Disaggregation}

\author{
Matt Wytock and J. Zico Kolter\\
School of Computer Science\\
Carnegie Mellon University\\
Pittsburgh, PA 15213 \\
\texttt{mwytock@cs.cmu.edu} \\
}

\DeclareMathOperator*{\minimize}{minimize}
\DeclareMathOperator*{\subjectto}{subject\;to}
\DeclareMathOperator*{\tr}{tr}

\nipsfinalcopy 

\begin{document}

\maketitle

\begin{abstract}
We propose a new framework for single-channel source separation that lies
between the fully supervised and unsupervised setting. Instead of
supervision, we provide input features
for each source signal and use convex methods to estimate the
correlations between these features and the unobserved signal decomposition. We
analyze the case of $\ell_2$ loss theoretically and show that recovery of
the signal components depends only on cross-correlation between features
for different signals, not on correlations between features for the same
signal. Contextually supervised source separation is a natural fit for
domains with large amounts of data but no explicit supervision; our motivating
application is energy disaggregation of hourly smart meter data (the separation
of whole-home
power signals into different energy uses). Here we apply contextual supervision to disaggregate the energy
usage of thousands homes over four years, a significantly larger scale than
previously published efforts, and demonstrate on synthetic data
that our method outperforms the unsupervised approach.
\end{abstract}

\section{Introduction}

We consider the \emph{single-channel source separation} problem, in
which we wish to separate a single aggregate signal into mixture of
unobserved component signals. Traditionally, this problem has been
approached in two ways: the \emph{supervised} setting
\cite{kolter2010energy,roweis2001one,schmidt2006single},
where we have access to training data with the true signal
separations and the \emph{unsupervised} (or ``blind'')
setting \cite{blumensath2005shift, davies2007source, lewicki2000learning,
  schmidt2006nonnegative}, where we have only the
aggregate signal.  However,
both settings have potential drawbacks: for many problems, including
energy disaggregation---which looks to separate individual energy uses
from a whole-home power signal \cite{hart1992nonintrusive}---it can
be difficult to obtain training data with the true separated signals
needed for the supervised setting; in contrast, the unsupervised
setting is an ill-defined problem with arbitrarily many solutions, and
thus algorithms are highly
task-dependent.

In this work, we propose an alternative approach that lies between
these two extremes. We propose a framework of \emph{contextual
  supervision}, whereby along with the input signal to be separated, we
provide contextual features correlated with the
unobserved component signals. In practice, we find that this is often
much easier than providing a fully supervised training set; yet it
also allows for a well-defined problem, unlike the unsupervised
setting. The approach is a natural fit for energy disaggregation, since
we have strong correlations between energy usage and easily observed
context---air conditioning spikes in hot summer months,
lighting increases when there is a lack of sunlight, etc. We
formulate our model directly as an optimization problem in which we jointly
estimate these correlations along with the most likely source
separation. Theoretically, we show that when the contextual features are
relatively uncorrelated between different groups, we can recover the
correct separation with high probability. We demonstrate that our
model recovers the correct separation on synthetic examples resembling
the energy disaggregation task and apply the
method to the task of separating sources of energy usage for thousands of homes
over 4 years, a scale much larger than previously published efforts.  The main
contributions of this
paper are 1) the proposed contextually supervised setting and the
optimization formulation; 2) the theoretical analysis showing that
accurate separation only requires linear independence between features
for different signals; and 3) the application of this approach to the
problem of energy disaggregation, a task with significant potential to
help inform consumers about their energy behavior, thereby increasing
efficiency.


\section{Related work}

As mentioned above, work in single-channel source separation
has been separated along the lines of supervised and
unsupervised algorithms, although several algorithms can be
applied to either setting.  A common strategy is to
separate  the observed aggregate signal into a linear combination
of several \emph{bases}, where different bases correspond to different
components of the signal; algorithms such as Probabilistic Latent
Component Analysis (PLCA) \cite{smaragdis2006probabilistic}, sparse coding
\cite{olshausen1997sparse}, and factorial hidden Markov models (FHMMs)
\cite{ghahramani1997factorial} all fall within this category, with
the differences concerning 1) how bases are represented and assigned
to different signal components and 2) how the algorithm infers the
activation of the different bases given the aggregate signal.  For
example, PLCA typically uses pre-defined basis functions (commonly
Fourier or Wavelet bases), with a probabilistic model for how sources
generate different bases; sparse coding learns bases tuned to data while
encouraging sparse activations;
and FHMMs use hidden Markov models to represent each source.  In the
supervised setting, one typically uses the individual signals to learn
parameters for each set of bases (e.g., PLCA will learn which bases
are typical for each signal), whereas unsupervised methods learn
through an EM-like procedure or by maximizing some separation
criteria for the learned bases.  The method we propose here is
conceptually similar, but the nature of these bases is rather different: instead
of fixed bases with changing activations, we require features that
effectively generate time-varying bases and learn
activations that are constant over time.

Orthogonal to this research, there has also been a great deal of work in
\emph{multi-channel} blind source separation problems, where we observe multiple
mixings of the same sources (typically, as many mixings as there are signals)
rather than in isolation.
These methods can exploit significantly more structure and algorithms like Independent Component Analysis
\cite{comon1994independent,bell1995information}
can separate signals with virtually no supervised information.
However, when applied to the single-channel problem (when this is even
possible), they typically perform substantially worse than
methods which exploit structure in the problem, such as those described above.

From the applied point of view, algorithms for energy disaggregation
have received growing interest in recently years
\cite{kolter2010energy,zeifman.11,kolter2012approximate,parson2012non}.
This is an important task since
many studies have shown that consumers naturally adopt energy
conserving behaviors when presented with a breakdown of their energy
usage \cite{darby.06,neenan.09,ehrhardt2010advanced}. Algorithmic
approaches to disaggregation are appealing as they allow for these
types of breakdowns, but existing disaggregation approaches virtually all use
high-frequency sampling of the whole-building power signal (e.g. per second)
requiring the installation of custom monitoring hardware for data collection.  In contrast,
this work focuses on disaggregation using data from ``smart meters'',
communication-enabled power meters that are currently installed in
more than 32 million homes \cite{iee.12}, but are limited to recording
usage at low frequencies (every 15 minutes or hour), leading to a substantially
different set of challenges. Smart meters are relatively new and deployment is
ongoing, but due to the large amount of data available now and in the near
future, successful disaggregation has the potential to have a profound impact on
energy efficiency.

\section{Optimization Formulation}

We begin by formulating the optimization problem for contextual source
separation. Formally, we assume there is some unknown matrix of $k$
component signals
\begin{equation}
Y \in \mathbb{R}^{T \times k} = \left [ \begin{array}{cccc}
\mid & \mid & & \mid \\ y_1 & y_2 & \cdots & y_k \\ \mid & \mid & &
\mid \end{array} \right ]
\end{equation}
from which we observe the sum $\bar{y} = \sum_{i=1}^k y_i$.  For
example, in our disaggregation setting, $y_i \in \mathbb{R}^T$ could denote  a
power trace (with $T$ total readings) for a single type of
appliance, such as the air conditioning, lighting, or electronics, and
$\bar{y}$ denotes the sum of all these power signals, which we observe
from a home's power meter.

In our proposed model, we represent each individual component signal
$y_i$ as a linear function of some component-specific bases $X_i \in
\mathbb{R}^{T \times n_i}$
\begin{equation}
y_i \approx X_i \theta_i
\end{equation}
where $\theta_i \in \mathbb{R}^{n_i}$ are the signal's
coefficients. The formal objective of our algorithm is: given the
aggregate signal $\bar{y}$ and the component features $X_i$,
$i=1,\ldots,k$, estimate both the parameters $\theta_i$ and the
unknown source components $y_i$.  We cast this as an optimization problem
\begin{equation}
\label{eq-opt}
\begin{split}
\minimize_{Y,\theta} \;\; & \sum_{i=1}^k \left \{ \ell_i(y_i, X_i\theta_i) +
g_i(y_i) + h_i(\theta_i) \right \} \\
\subjectto \;\; & \sum_{i=1}^k y_i = \bar y \\
\end{split}
\end{equation}
where $\ell_i : \mathbb{R}^{T} \times \mathbb{R}^{T} \rightarrow
\mathbb{R}$ is a loss function penalizing differences between the
$i$th reconstructed signal and its linear representation; $g_i$ is a
regularization term encoding the ``likely'' form of the signal $y_i$,
independent of the features; and $h_i$ is a regularization penalty on
$\theta_i$. Choosing $\ell_i$, $g_i$ and $h_i$ to be convex functions
results in a convex optimization problem.

A natural choice of loss function $\ell_i$ is a norm
penalizing the difference between the reconstructed signal and its features
$\|y_i - X_i \theta_i\|$, but since our formulation enables loss functions that
depend simultaneously on all $T$ values of the signal, we allow for more complex
choices as well.  For example
in the energy disaggregation problem, air conditioning is correlated with
high temperature but does not respond to outside temperature changes instantaneously;
thermal mass and the varying occupancy in buildings often results in air
conditioning usage that correlates with high temperature over some
window (for instance, if no one is in a room during a period of high
temperature, we may not use electricity then, but need to ``make up'' for this
later when someone does enter the room).  Thus, the loss function
\begin{equation}
\ell_i(y_i, X_i \theta_i) = \|(y_i - X_i \theta_i)(I \otimes
1^T)\|^2_2
\end{equation}
which penalizes the aggregate difference of $y_i$ and $X_i\theta_i$ over a
sliding window, can be used to capture such dynamics. In many settings, it
may also make sense to use $\ell_1$ or $\ell_\infty$ rather than $\ell_2$ loss,
depending on the nature of the source signal.

Likewise, since the objective term $g_i$ depends on all $T$ values of $y_i$, we
can use it to encode the likely dynamics of the source signal independent of
$X_i\theta_i$. For air conditioning and other single appliance types, we expect
transitions between on/off
states which we can encode by penalizing the $\ell_1$ norm of $Dy_i$ where
$D$ is the linear difference operator subtracting $(y_i)_{t-1} - (y_i)_t$. For
other types of energy consumption, for example groups of many electronic
appliances, we expect the signal to have smoother
dynamics and thus $\ell_2$ loss is more appropriate.  Finally, we also include
$h_i$ for statistical regularization purposes---but for problems where $T \gg
n_i$, such as the ones we consider in our experiments, the choice of $h_i$ is
less important.

\section{Theoretical analysis}

Next, we consider the ability of our model to recover the true source
signals as $T$ grows while $k$ and $n_i$ remain fixed. For the purposes
of this section only, we restrict our attention to the choice of
$\ell_2$ loss, no $g_i$ or $h_i$ terms, and
Gaussian noise (the extension to the sub-Gaussian case is
straightforward).  We show that under this specialization of
the model, the optimization problem recovers the underlying signals
at a rate dependent on the linear independence between blocks of input features
$X_i$. In practice, the choice of $\ell_i$, $g_i$ and
$h_i$ is problem-specific, but $\ell_2$ loss is a reasonable default
and while simplifying the theoretical analysis dramatically, captures
the essential behavior of the model in the large $T$ regime.

Formally, for this section we assume the source signals have Gaussian noise
\begin{equation}
\label{eq-model}
y_i = X_i\theta^\star_i + w_i
\end{equation}
for some $\theta^\star_i \in \mathbb{R}^{n_i}$ and $w_i \sim
\mathcal{N}(0, \sigma_i^2 I)$.  Under the choice of $\ell_2$ loss, our
optimization problem simplifies to
\begin{equation}
\begin{split}
\minimize_{Y,\theta} \;\; &\|Y1 - X\theta\|_2^2 \\
\subjectto \;\; & Y1 = \bar y
\end{split}
\end{equation}
where $1 \in \mathbb{R}^k$ is the all-ones vector, $\theta \in
\mathbb{R}^{n}$ is a concatenation of all the $\theta_i$'s, $X \in
\mathbb{R}^{T \times n}$ is a concatenation of all the
$X_i$'s  and $n = \sum_{i=1}^k n_i$ is the total number of features. In
particular, estimation of $\theta$ is equivalent to
least-squares
\begin{equation}
\label{eq-theta-estimate}
\hat{\theta} \in \mathbb{R}^{n} = \arg \min_{\theta} \|\bar y - X
\theta\|_2^2 = (X^T X)^{-1} X^T{\bar y}.
\end{equation}
Since each $y_i$ has it's own noise term, we can never expect to recover $y_i$
exactly, but we can recover the true $\theta^\star$ with analysis that is the same
as for standard linear regression. However, in the context of source separation,
we are interested in the recovery of the ``noiseless'' $y_i$, $X_i
\theta^\star_i$,  and thus in our analysis we consider how the root mean
squared error
\begin{equation}
\label{eq-err}
\mathrm{RMSE}(X_i\hat{\theta}_i) = \sqrt{\frac{1}{T} \left\|X_i
  \hat{\theta}_i - X_i \theta_i^\star \right \|_2^2}
\end{equation}
vanishes for large $T$; indeed, a key feature of our analysis is to
show that we may recover the underlying signals faster than $\theta^\star_i$.

\begin{theorem}
\label{theorem-bounds}
Given data generated by the model \eqref{eq-model}, and estimating
$\hat{\theta}$ via \eqref{eq-theta-estimate}, we have that
\begin{equation}
\label{eq-expected-value}
\mathbf{E}\left[\|X_i \hat{\theta}_i - X_i \theta^\star\|_2^2 \right ]
= \sigma^2 \tr X_i^T X_i (X^T X)^{-1}_{ii} \leq \sigma^2 n_i
\rho_i
\end{equation}
where $\sigma^2 = \sum_{i=1}^k \sigma_i^2$ and $\rho_i =
\lambda_{\max}(X_i^T X_i (X^T X)^{-1}_{ii})$.  Furthermore, for $\delta
\leq 0.1$, with probability greater than $1 - \delta$
\begin{equation}
\label{eq-sample-complexity}
\mathrm{RMSE}(X_i\hat{\theta}_i) \leq \sqrt{\frac{4 \sigma^2 n_i
    \rho_i \log (1/\delta)}{T}}.
\end{equation}
\end{theorem}

A key quantity in this theorem is the matrix $X_i^T X_i (X^T
X)^{-1}_{ii} \in \mathbb{R}^{n_i \times n_i}$; $(X^T X)^{-1}_{ii}$
denotes the $i,i$ block of the full inverse $(X^TX)^{-1}$ (i.e., first
inverting the joint covariance matrix of all the features, and then
taking the $i,i$ block), and this term provides a measure of the
linear independence between features corresponding to \emph{different}
signals.   To see this, note that if features across different signals
are orthogonal, $X^T X$ is block diagonal, and thus $X_i^T X_i (X^T
X)^{-1}_{ii} = X_i^T X_i (X_i^T X_i)^{-1} = I$, so $\rho_i = 1$.
Alternatively, if two features provided for different signals
are highly correlated, entries of $(X^TX)_{ii}^{-1}$ will have large
magnitude that is not canceled by $X_i^T X_i$ and $\rho_i$ will be large.  This
formalizes an intuitive notion for
contextually supervised source separation: for recovering the
underlying signals, it does not matter if two features for the
\emph{same} signal are highly correlated (this contrasts to the case
of recovering $\theta^\star$ itself which depends on all correlations), but two
correlated signals for different features make estimation difficult;
intuitively, if two very similar features are provided for two different source
signals, attribution becomes difficult. A particularly useful property of these
bounds  is that all terms can be computed using just $X_i$, so we can estimate
recovery rates when choosing our design matrix.

\subsection{Proofs}

The proof of Theorem \ref{theorem-bounds} proceeds in two steps.
First, using rules for linear transformations of Gaussian random
variables, we show that the quantity $X_i(\hat{\theta}_i -
\theta^\star_i)$ is also (zero-mean) Gaussian, which immediately leads
to \eqref{eq-expected-value}.  Second, we derive a tail bound on the
probability that $X_i(\hat{\theta}_i - \theta^\star_i)$ exceeds some
threshold, which leads to the sample complexity bound
\eqref{eq-sample-complexity}; because this quantity has a singular
covariance matrix, this requires a slightly specialized probability
bound, given by the following lemma.

\begin{lemma}
\label{lemma-tail}
Suppose $x \in \mathbb{R}^p \sim \mathcal{N}(0,\Sigma)$ with
$\mathrm{rank}(\Sigma) = n$.  Then
\begin{equation}
P\left(\|x\|_2^2 \geq t\right) \leq \left (\frac{t}{n \lambda} \right )^{n/2}
\exp\left \{-\frac{1}{2}(t/\lambda - n) \right \}
\label{eq-syn}
\end{equation}
where $\lambda$ is the largest eigenvalue of $\Sigma$.
\end{lemma}

\begin{proof}
By Chernoff's bound
\begin{equation}
P\left(\|x\|_2^2 \geq t\right) \leq
\frac{\mathbf{E}\left[e^{\alpha \|x\|_2^2}\right ]}{e^{\alpha t}}.
\end{equation}
for any $\alpha \geq 0$.  For any $\epsilon > 0$, $z \sim
\mathcal{N}(0, \Sigma + \epsilon I)$,
\begin{equation}
\begin{split}
\mathbf{E}\left[e^{\alpha \|z\|_2^2}\right] & = (2\pi)^{-p/2}|\Sigma + \epsilon
  I|^{-1/2}  \int \exp\left \{ - \frac{1}{2} z^T (\Sigma + \epsilon I)^{-1}
  z + \alpha z^T z \right \} dz \\
& = (2\pi)^{-p/2}|\Sigma + \epsilon
  I|^{-1/2}  \int \exp\left \{ - \frac{1}{2} z^T (\Sigma + \epsilon
  I)^{-1}(I - 2 \alpha (\Sigma + \epsilon I)^{-1})
  z \right \} dz \\
& = (2\pi)^{-p/2}|\Sigma + \epsilon I|^{-1/2} (2\pi)^{p/2}|\Sigma +
  \epsilon I|^{1/2} |I - 2 \alpha (\Sigma + \epsilon I)|^{-1/2}
\end{split}
\end{equation}
so taking the limit $\epsilon \rightarrow 0$, we have that
$\mathbf{E}[e^{\alpha \|x\|_2^2}] = |I - 2 \alpha
\Sigma|^{-1/2}$.  Since $\Sigma$ has only $n$
nonzero eigenvalues,
\begin{equation}
|I - 2 \alpha \Sigma| = \prod_{i=1}^n (1 - 2 \alpha \lambda_i)
\geq (1 - 2 \alpha \lambda)^n
\end{equation}
and so
\begin{equation}
P\left(\|x\|_2^2 \geq t\right) \leq \frac{1}{(1 - 2 \alpha
  \lambda)^{n/2}e^{\alpha t}}.
\end{equation}
Minimizing this expression over $\alpha$ gives $\alpha = \frac{t - n
  \lambda}{2t\lambda}$ and substituting this into the equation
above gives the desired bound.
\end{proof}

To write the problem more compactly, we define the block matrix $W \in
\mathbb{R}^{T \times k}$ with columns $w_i$, and define the
``block-diagonalization'' operator $B : \mathbb{R}^{n} \rightarrow
\mathbb{R}^{n \times k}$ as
\begin{equation}
B(\theta) = \left [ \begin{array}{cccc}
\theta_1 & 0 & \cdots & 0 \\
0 & \theta_2 & \cdots & 0 \\
\vdots & \vdots & \ddots & \vdots \\
0 & 0 & \cdots & \theta_k \end{array} \right ].
\end{equation}

\begin{proof}
(of Theorem \ref{theorem-bounds}) Since $Y = X B(\theta^\star) + W$,
\begin{equation}
\begin{split}
X B(\hat{\theta}) & = X B\left((X^T X)^{-1} X^T (X B(\theta^\star) +
W)1 \right) \\ & = X B(B(\theta^\star)) 1 + X B\left((X^T X)^{-1} X^T W 1
\right) \\ & = X B(\theta^\star) + X B\left((X^T X)^{-1} X^T W 1
\right)
\end{split}
\end{equation}
For simplicity of notation we also denote
\begin{equation}
u \in \mathbb{R}^{n} \equiv (X^T X)^{-1} X^T W 1.
\end{equation}
Thus
\begin{equation}
X B(\theta^\star) - X B(\hat{\theta})  = X B(u).
\end{equation}
Now, using rules for Gaussian random variables under linear transformations
$W 1 \sim \mathcal{N}(0, \sigma^2 I_T)$ and so $u \sim \mathcal{N}(0,
\sigma^2(X^T X)^{-1})$.  Finally, partioning $u_1, u_2, \ldots,
u_k$ conformally with $\theta$,
\begin{equation}
X_i \theta^\star - X_i \hat{\theta}_i = X_i u_i \sim \mathcal{N}(0,
\sigma^2 X_i(X^T X)^{-1}_{ii} X_i^T)
\end{equation}
so
\begin{equation}
\mathbf{E} \left [\left \|X_i \theta^\star - X_i \hat{\theta}_i \right
  \|_2^2 \right ] = \sigma^2  \tr X_i^T X_i(X^T X)^{-1}_{ii}.
\end{equation}

Since $\sigma^2 X_i(X^T X)^{-1}_{ii}X_i^T$ is a rank $n_i$ matrix
with maximum eigenvalue equal to $\sigma^2 \rho_i$, applying
Lemma \ref{lemma-tail} above gives
\begin{equation}
P\left ( \mathrm{RMSE}(X_i\hat{\theta}_i) \geq \epsilon \right) =
P\left ( \|X_iu_i\|^2_2 \geq T \epsilon^2 \right) \leq \left
  (\frac{T \epsilon^2}{n_i \sigma^2 \rho_i} \right )^{n_i/2} \exp\left
  \{-\frac{1}{2}\left (\frac{T \epsilon^2}{\sigma^2 \rho_i} -
  n_i\right) \right \}.
\end{equation}
Setting the right hand side equal to $\delta$ and solving for
$\epsilon$ gives
\begin{equation}
\epsilon = \sqrt{\frac{-W(-\delta^{2/n}/e) n_i \rho_i \sigma^2}{T}}
\end{equation}
where $W$ denotes the Lambert $W$ function (the inverse of $f(x) = x
e^x$).  The theorem follows by noting that $-W(-\delta^{2/n}/e)
\leq 4 \log \frac{1}{\delta}$ for all $n \geq 1$ when $\delta \leq
0.1$, with both quantities always positive in this range (note that
leaving the $W$ term in the bound can be substantially tighter in some
cases).
\end{proof}

\section{Experimental results}

\begin{figure}
\includegraphics{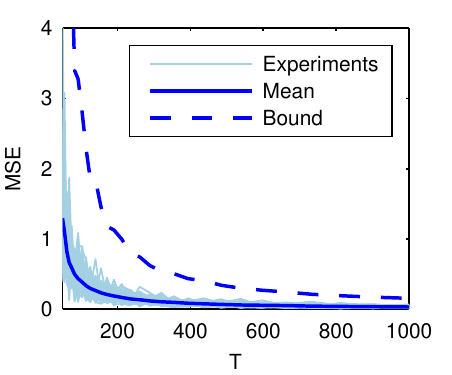}
\includegraphics{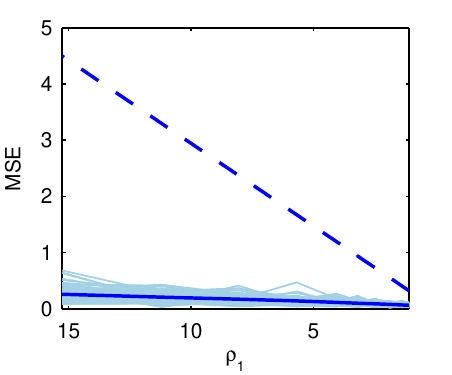}
\includegraphics{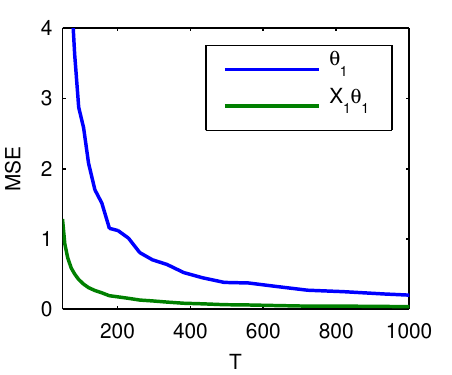}
\caption{Comparing 100 random experiments with the theoretical
  mean and 90\% upper bound for increasing $T$ (left); fixing $T = 500$ and
  varying $\rho_i$ (center); and compared to the recovery of
  $\theta_1$ directly (right).}
\label{fig-syn1}
\end{figure}

In this section we evaluate contextual supervision on synthetic
data and apply in to disaggregate smart meter data collected from thousands of
homes. Since labeled data is unavailable, we design a synthetic
disaggregation task similar to energy disaggregation from smart meters in order
to evaluate our performance on this task quantitatively. Here we explore the
choice of loss functions and demonstrate that contextual supervision
dramatically outperforms the unsupervised approach.

\textbf{Rates for signal recovery}. We begin with a set of experiments examining
the ability of the model to recover the
underlying source signals as predicted by our theoretical analysis. In these
experiments, we consider the problem of separating two source signals with $X_i
\in \mathbb{R}^{T \times 16}$
sampled independently from a zero-mean Normal with covariance $I + (1-\mu)11^T$; we sample $\theta_i^\star$ uniformly from $[-1,1]$ and $Y \sim
\mathcal{N}(X\theta^\star, I)$. Setting $\mu = 0.01$ causes the
features for each signal to be highly correlated with each other but since $X_i$
are sampled independently, not highly correlated across signals. In Figure \ref{fig-syn1},
we see that MSE vanishes as $T$ grows; when comparing these experimental results to
the theoretical mean and 90\% upper bound, we see that at least in these
experiments, the bound is somewhat loose for large values of $\rho_i$. We are
also able to recover $\theta^\star$ (which is expected since the least-squares
estimator $\hat \theta$ is consistent), but the rate is much slower due to
high correlations in $X_i$.

\textbf{Disaggregation of synthetic data}. The next set of experiments considers
a synthetic generation process that more
closely mimics signals that we encounter in energy disaggregation. The process
described visually in Figure
\ref{fig-syn2} (top) begins with two signals, the first is smoothly varying over
time while the other is a repeating step function
\begin{equation}
X_1(t) = \sin(2\pi t / \tau_1) + 1,  \quad X_2(t) = I(t \bmod \tau_2 <
\tau_2/2)
\end{equation}
where $I(\cdot)$ is the indicator function and $\tau_1$, $\tau_2$ are the
period of each signal. We also use
two different noise models: for the smooth signal we sample Gaussian noise from
$\mathcal{N}(0, \sigma^2)$ while for the step function,
we sample a distribution with a point mass at zero, uniform probability
over $[1, 0) \cup (0, 1]$ and correlate it across time by summing over a
window of size $\beta$. Finally, we constrain both noisy signals to be
nonnegative and sum them to generate our input.

\begin{figure}
\includegraphics{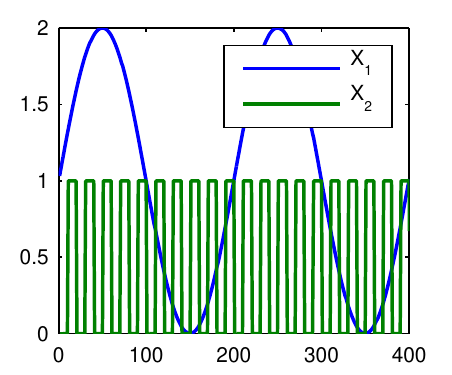}
\includegraphics{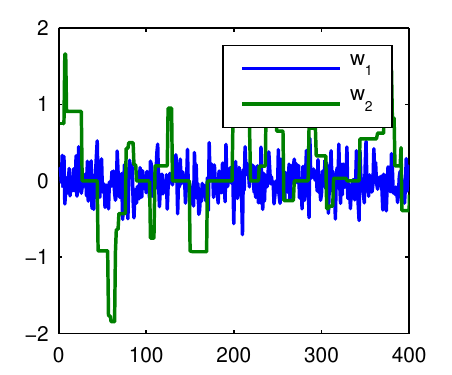}
\includegraphics{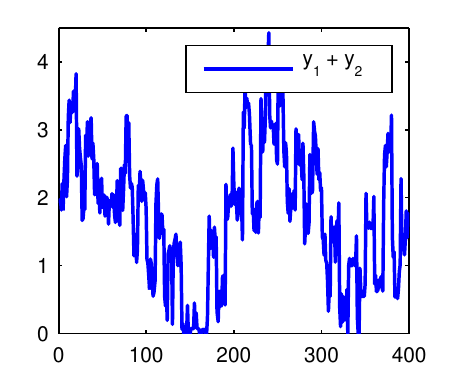}
\includegraphics{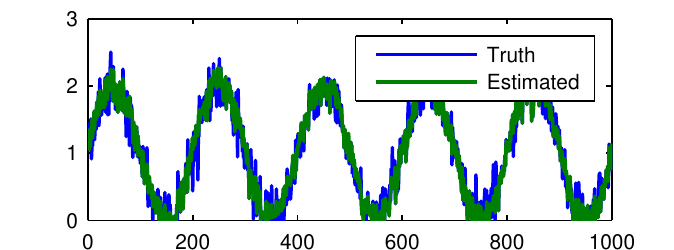}
\includegraphics{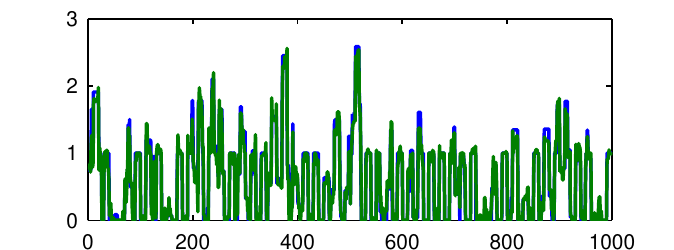}
\caption{(top) Synthetic data generation process starting with two underlying
  signals (left), corrupted by different noise models (center), and constrained
  to be nonnegative
  and sum'd to give the observed input (right); (bottom) comparison of the true noisy
  signal and estimates.}
\label{fig-syn2}
\end{figure}

We generate data under this model for $T = 50000$ time points and consider
increasingly specialized optimization objectives while measuring the error in
recovering $Y^\star = XD(\theta^\star) + W$, the underlying source signals
corrupted by noise. As can be seen in Table \ref{tab-syn2-mse}, by using $\ell_1$
loss for $y_2$ and adding $g_i(y_i)$ terms
penalizing $\|Dy_1\|_2^2$ and $\|Dy_2\|_1$, error decreases by 25\% over
just $\ell_2$ loss alone; in Figure \ref{fig-syn2}, we observe that our
estimations recovers the true source signals closely with the $g_i$ terms
helping to capture the dynamics of the noise model for $w_2$.

As a baseline for this result, we compare to an unsupervised method, nonnegative
sparse coding \cite{hoyer2002non}. We apply sparse coding
by segmenting the input signal into $1000$ examples of $50$ time points
(1/4 the period of the sine wave, $X_1(t)$) and fit a sparse
model of 200 basis functions. We report the best possible source separation by
assigning each basis function according to an oracle measuring
correlation with the true source signal and using the best value over a grid of
hyperparameters; however, performance
is still significantly worse than the contextually supervised method which makes
explicit use of additional information.
\begin{table}
\label{tab-syn2-mse}
\centering
\caption{Comparison of models for source separation on synthetic data.}
\begin{tabular}{|l|r|}
\hline
\textbf{Model} & \textbf{RMSE} \\
\hline
Nonnegative sparse coding & 0.4035 \\
$\ell_2$ loss for $y_1$ & 0.1640 \\
$\ell_2$ loss for $y_1$, $\ell_1$ loss for $y_2$ & 0.1520 \\
$\ell_2$ loss for $y_1$, $\ell_1$ loss for $y_2$ and $g_i$ penalizing $\|Dy_i\|$  & \textbf{0.1217} \\
\hline
\end{tabular}
\end{table}

\textbf{Energy disaggregation on smart meter data}. Next, we turn to the motivating problem for our model: disaggregating
large-scale low resolution smart meter data into its component sources of
consumption. Our dataset comes from PG\&E and was collected by the utility from
customers in Northern California who had smart meters between 1/2/2008
and 12/31/2011. According to estimations based on survey data, heating and cooling (air
conditioning and refrigerators) comprise over 39\% of total consumer
electricity usage \cite{recs} and thus are dominant uses for
consumers. Clearly, we expect temperature to have a strong correlation with
these uses and thus we provide contextual supervision in the form of
temperature information. The PG\&E data is anonymized, but the
location of individual customers is identified at the census block
level; we use this information to construct a parallel temperature
dataset using data from Weather Underground
(\url{http://www.wunderground.com/}).

The exact specification of our energy disaggregation model is given in Table
\ref{tab-dis-features}---we capture the non-linear dependence on temperature
with radial-basis functions (RBFs), include a ``Base'' category which models
energy used as a function of time of day, and featureless ``Other'' category
representing end-uses not explicitly modeled. For simplicity, we penalize each
category's deviations from the model using $\ell_1$ loss; but for heating and
cooling we first multiply by a smoothing matrix $S_n$ ($1$'s on
the diagonal and $n$ super diagonals) capturing the thermal mass
inherent in heating and cooling: we expect energy usage to correlate
with temperature over a window of time, not immediately. Finally, we use $g_i(y_i)$ and the
difference operator to encode our intuition of how energy consumption in each
category evolves over time. The ``Base'' category represents an aggregation of
many sources of consumption and which we expect to evolve smoothly over time,
while the on/off behavior in other categories is best represented by the
$\ell_1$ penalty.

We present the result of our model at two time scales, starting with Figure 3
(top), where we show aggregate energy consumption across all homes at the
week level to demonstrate basic trends in usage. Quantitatively, our model
assigns 15.6\% of energy consumption to ``Cooling'' and 7.7\% to
``Heating'', which is reasonably close to estimations based on survey
data \cite{recs} (10.4\% for air conditioning and 5.4\% for space heating). We
have deliberately kept the model simple and thus our higher estimations are
likely due to conflating other temperature-related energy usages, such as refrigerators and
water heating. In Figure \ref{fig-dis} (bottom), we present the
results of the model in disaggregating energy usage for a single hot summer
week where the majority of energy usage is estimated to be cooling due
to the context provided by high temperature.

\begin{table}
\label{tab-dis-features}
\centering
\caption{Model specification for contextually supervised energy disaggregation.}
\begin{tabular}{|l|l|l|l|}
\hline
\textbf{Category} & \textbf{Features} & \textbf{$\ell_i$} & \textbf{$g_i$} \\
\hline
Base    & Hour of day & $\|y_1 - X_1\theta_1 \|_1 $ & $\|Dy_1 \|_2^2$   \\
Cooling & RBFs over temperatures $>70^\circ\mathrm{F}$ & $\|S_2(y_2 - X_2\theta_2) \|_1$ &
$0.1 \times \| Dy_2\|_1$ \\
Heating & RBFs over temperatures $<50^\circ\mathrm{F}$ & $\|S_2(y_3 - X_3\theta_3) \|_1$ &
$0.1 \times \| Dy_3 \|_1$ \\
Other   & None & $\| y_4 \|_1 $ & $0.05 \times \| Dy_4 \|_1$ \\
\hline
\end{tabular}
\end{table}

\begin{figure}
\includegraphics{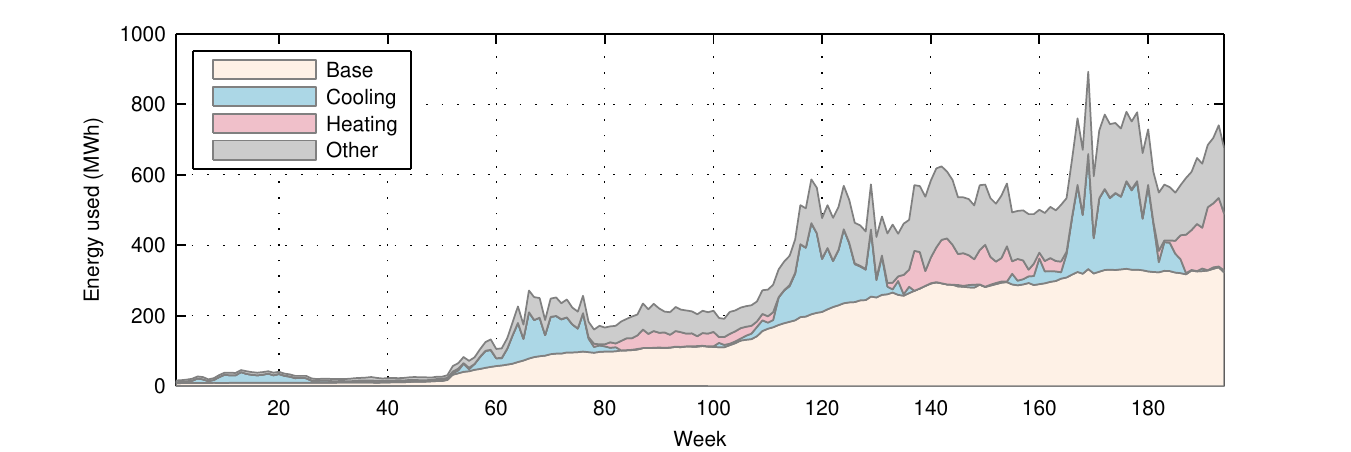}
\includegraphics{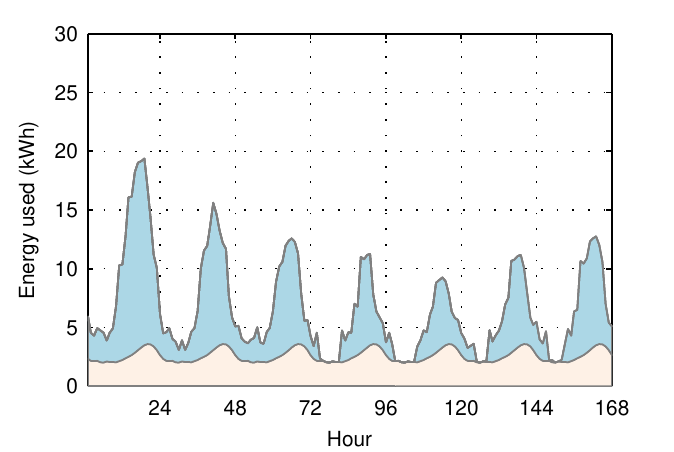}
\includegraphics{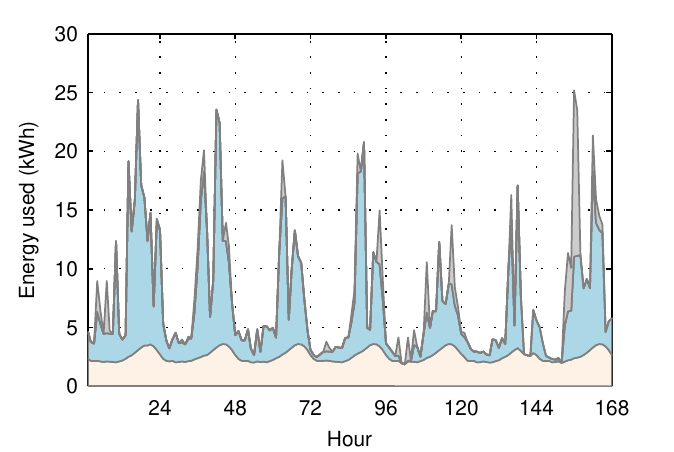}
\caption{Disaggregated energy usage for 1276 homes over nearly four years shown weekly
  (top); and for a single home near Fresno, California over the week starting 6/28/2010 (bottom) with
  estimated correlations $X_i\hat{\theta}_i$ (left) and estimated energy uses $\hat{y}_i$ (right).}
\label{fig-dis}
\end{figure}

\section{Conclusion and discussion}

We believe the advances in this work, formalizing contextually supervised source
separation and theoretically analyzing the requirements for accurate source
signal recovery, will enable new applications of single-channel source separation
in domains with large amounts of data but no access to explicit supervision. In
energy disaggregation, this approach has allowed us to reasonably separate
sources of consumption from extremely low-frequency smart meter data. This a
significant advancement with the potential to drive increases in energy
efficiency through programs that expose this information to consumers and
automated systems for demand response. Developing algorithms that use this
information to achieve these goals is an interesting direction for
future work.

Another interesting direction is the explicit connection of our large-scale
low-resolution methods with the more sophisticated appliance models developed on
smaller supervised datasets with high-frequency measurements. As a few examples, a
small amount of supervised information would enable us to calibrate
our models automatically, while a semi-supervised approach would enable
spreading some of the benefit of high-resolution load monitoring to the vast
number of homes where only smart meter data is available.

\bibliographystyle{abbrv}
\bibliography{context}

\end{document}